\newcommand{\longversion}[1]{#1}
\newcommand{\shortversion}[1]{}

\documentclass{tlp}

\usepackage[T1]{fontenc}
\usepackage[utf8]{inputenc}

\usepackage{amsmath,amsfonts,amssymb,mathtools}
\usepackage{graphicx}
\usepackage{multirow}
\usepackage{subcaption}

\usepackage[english]{babel}

\usepackage[noend]{algorithmic}
\usepackage{graphicx}
\usepackage{tabularx}
\usepackage[table]{xcolor}
\usepackage{booktabs}
\usepackage{hyperref}

\usepackage{microtype}

\newcommand{\complClFont}[1]{\mathrm{#1}}
\newcommand{\problemFont}[1]{\textsc{#1}}
\newcommand{\mathCommandFont}[1]{\mathrm{#1}}

\let\proof\undefined

\RequirePackage{amsthm}
\usepackage{url}

\newtheorem{definition}{Definition}[section]
\newtheorem{example}[definition]{Example}

\let\phi\varphi

\DeclareMathOperator\var{var}

\newcommand{\PMABD}{\protect\ensuremath\problemFont{ABD}}
\newcommand{\divABD}{\protect\ensuremath\problemFont{Div-ABD}}
\newcommand{\divPosTwoSAT}{\protect\ensuremath\problemFont{Div-Pos2SAT}}

\newcommand{\SAT}{\protect\ensuremath\problemFont{SAT}}

\newcommand{\KB}{\protect\ensuremath\text{KB}}

\newcommand{\cclone}[1]{\ensuremath{\langle #1 \rangle}}

\newcommand{\str}[1]{\ensuremath{\mathrm{Str}}}

\newcommand{\clos}[1]{\left\langle #1 \right\rangle}
\newcommand{\closneq}[1]{\left\langle #1 \right\rangle_{\neq}}

\newcommand{\cloneFont}[1]{\mathsf{#1}}
\newcommand{\BR}{\protect\ensuremath{\cloneFont{BR}}}
\newcommand{\II}{\protect\ensuremath{\cloneFont{II}}}
\newcommand{\IL}{\protect\ensuremath{\cloneFont{IL}}}
\newcommand{\IS}[2]{\protect\ensuremath{\cloneFont{IS}^{#1}_{#2}}}
\newcommand{\ID}{\protect\ensuremath{\cloneFont{ID}}}
\newcommand{\IN}{\protect\ensuremath{\cloneFont{IN}}}
\newcommand{\IV}{\protect\ensuremath{\cloneFont{IV}}}
\newcommand{\IE}{\protect\ensuremath{\cloneFont{IE}}}
\newcommand{\IM}{\protect\ensuremath{\cloneFont{IM}}}
\newcommand{\IR}{\protect\ensuremath{\cloneFont{IR}}}
\newcommand{\IBF}{\protect\ensuremath{\cloneFont{IBF}}}

\newcommand{\epos}{\protect\ensuremath{\cloneFont{EP}}}
\newcommand{\eneg}{\protect\ensuremath{\cloneFont{EN}}}
\newcommand{\pos}{\protect\ensuremath{\cloneFont{P}}}
\newcommand{\negg}{\protect\ensuremath{\cloneFont{N}}}
\newcommand{\Card}[1]{\ensuremath{|\,#1\,|}}

\usepackage{tcolorbox} 
  
 \newcommand{\problemDef}[3]
{%
    
    \begin{tcolorbox}[arc=0.1mm,boxsep=-0.6mm,left=1.9mm,right=1.9mm,bottom=1.4mm,top=1.4mm,adjusted title={\strut \sc#1},colback=white!5]

    \noindent\textbf{Given:} #2.
    
    \noindent\textbf{Task:} #3
    \end{tcolorbox}
}

\usepackage{xspace}

\newcommand{\isfacet}{\textsc{IsFacet}\xspace}

\newcommand{\NP}{\protect\ensuremath{\complClFont{NP}}\xspace} 
\newcommand{\SigmaP}{\protect\ensuremath{{\Sigma}^{\Ptime}_2}}

\newcommand{\Ptime}{{\protect\ensuremath{\complClFont{P}}}\xspace}

\newcommand{\co}{\text{co}}

\usepackage{mathtools}
\renewcommand{\mid}{\;|\;}
\newcommand{\eqdef}{\coloneqq}

\newcommand{\citex}[1]{\citeauthor{#1}~(\citeyear{#1})}
\renewcommand{\cite}[2][]{\citep[#1]{#2}}

\let\oldpar\paragraph
\renewcommand{\paragraph}[1]{\oldpar{#1{}.}}

\usepackage[ruled,vlined,linesnumbered]{algorithm2e}

\SetKwInput{KwData}{In}
\SetKwInput{KwResult}{Out}
\SetAlFnt{\small}
\SetAlCapFnt{\small}
\SetAlCapNameFnt{\small}
\SetAlCapHSkip{0pt}
\SetEndCharOfAlgoLine{}

\makeatletter
\newcommand{\algorithmfootnote}[2][\footnotesize]{
  \let\old@algocf@finish\@algocf@finish
  \def\@algocf@finish{\old@algocf@finish
    \leavevmode\rlap{\begin{minipage}{\linewidth}
    #1#2
    \end{minipage}}
  }
}
\makeatother

\usepackage{thmtools}
\usepackage{thm-restate}
\usepackage{apptools}

\declaretheorem[name=Lemma,sibling=definition]{lemma}
\declaretheorem[name=Theorem,sibling=definition]{theorem}

\declaretheorem[name=Proposition,sibling=definition]{prop}

\begin{document}
\label{firstpage}

\lefttitle{J.~Schmidt, M.~Maizia, V.~Lagerkvist, J.~Fichte}

\jnlPage{1}{8}
\jnlDoiYr{2025}
\doival{10.1017/xxxxx}

\title[Complexity of Faceted Explanations in Propositional
Abduction]{%
  Complexity of Faceted Explanations\\
  in Propositional Abduction%
  \thanks{Author names are stated in reverse alphabetical
    order. 
    \longversion{%
      This is an extended version of a paper~\cite{SchmidtMaiziaLagerkvist25}
      that has been accepted to the Proceedings of the 41st
      International Conference on Logic Programming (ICLP'25).
    }
  }
}

\begin{authgrp}
\author{\sn{Johannes} \gn{Schmidt}} %
\affiliation{Jönköping University %
  \email{johannes.schmidt@ju.se}
} 

\author{\sn{Mohamed} \gn{Maizia}}
\affiliation{Jönköping University, Linköping University%
    \email{mohamed.maizia@ju.se}
}
\author{\sn{Victor} \gn{Lagerkvist}} %
\affiliation{Linköping University %
  \email{victor.lagerkvist@liu.se}
}
\author{\sn{Johannes K.} \gn{Fichte}}
\affiliation{Linköping University
  \email{johannes.fichte@liu.se}
}
\end{authgrp}

\maketitle

\begin{abstract}
Abductive reasoning is a popular non-monotonic paradigm that aims to explain observed symptoms and manifestations. It has many applications, such as diagnosis and planning in artificial intelligence and database updates. In propositional abduction, we focus on specifying knowledge by a propositional formula. The computational complexity of tasks in propositional abduction has been systematically characterized -- even with detailed classifications for Boolean fragments.
Unsurprisingly, the most insightful reasoning problems (counting and enumeration) are computationally highly challenging. Therefore, we consider reasoning between decisions and counting, allowing us to understand explanations better while maintaining favorable complexity. 
We introduce facets to propositional abductions, which are literals that occur in some explanation (relevant) but not all explanations (dispensable). Reasoning with facets provides a more fine-grained understanding of variability in explanations (heterogeneous). In addition, we consider the distance between two explanations, enabling a better understanding of heterogeneity/homogeneity.
We comprehensively analyze facets of propositional abduction in various settings, including an almost complete characterization in Post's framework.
\end{abstract}


\section{Introduction} \label{sec:intro}
    Pedro is a passionate sailor. Today is Wednesday and he is about to go on the usual race to enjoy some waves and get a good challenge. But for some reason, nobody is out there for a race (a {\em manifestation}). He looks up into the sky and is a bit undecided, could the weather forecast have predicted calm winds or an unexpected storm resulting in calling off the race (a {\em hypothesis})
    trying to explain his observation by finding appropriate causes.
    This type of backward reasoning is called abductive reasoning, one of the fundamental reasoning techniques that is commonly believed to be naturally used by humans when searching for diagnostic explanations.
    Abduction has many applications~\cite{Dellsen_2024,Wang-ZhouEtAl19,IgnatievNarodytskaMarques19,YuEtAl2023,HuEtAl2025,YuEtAl2023} and is well-studied in the areas of artificial intelligence, knowledge representation, and non-monotonic reasoning~\cite{Miller19,Kakas92,Minsky74}.

    Qualitative reasoning problems like deciding whether an explanation exists or whether a proposition is relevant or necessary are computationally hard but still within range of modern solving approaches. More precisely, these problems are located on the second level of the polynomial hierarchy in the general case~\cite{EiterGotlob95}. 
    However, asking for relevance or necessary propositions does not provide much insight into the variability of explanations.
    Enumeration and counting allow for more fine-grained reasoning but are computationally extremely expensive~\cite{HermannPichler10,CreignouEtAl19}.
    Instead, we turn our attention to the world between \emph{relevant} (belongs to some explanations) and \emph{necessary} propositions (belongs to all explanations) and consider propositions that  are \emph{relevant but not necessary}, called \emph{facets}.

    In this work, we study the computational complexity of problems involving facets. To this end, we work in the universal algebraic setting by restricting the types of clauses/relations that are allowed (e.g., only Horn-clauses, or only 2-CNF). The resulting sets can be described by functions called {\em polymorphisms} and in the Boolean domain form a lattice known as {\em Post's lattice}~\cite{Post41}. 
    This setting makes it possible to obtain much more fine-grained complexity results and two prominent and early results are Lewis' dichotomy for propositional satisfiability~\cite{Lewis79} and Schaefer's dichotomy for Boolean constraint satisfaction~\cite{Schaefer78}. However, this approach has been applied to many more problems~\cite{Creignou2008}, including non-monotonic reasoning and several variants of abduction~\cite{NordhZ08}.
    We follow this line of research for facetted abduction.

    \smallskip
    \noindent Our \textbf{main contributions} are the following:\\[-1.5em]
    \begin{enumerate}
    \item We introduce facets to propositional abduction thereby 
    enabling a better understanding of propositions in  %
    explanations.
    \item We establish a systematic complexity characterization  for deciding facets in propositional abduction illustrated in Figure~\ref{fig:isfacet}.\\[-1.75em]
    \begin{enumerate}
        \item[(i)]
    Our classification provides a \emph{complete picture} in Post's framework for all fragments, which can be described via clauses,~e.g., Horn, 2CNF, dualHorn. Only two open cases remain: relations definable as Boolean linear equations of even length (with, or without, unit clauses).  
    \item[(ii)]
    Deciding facets is often not much harder than deciding explanation existence.
    In some surprising cases the complexity %
    increases. 
    \item[(iii)] 
    Our facet results imply a corresponding classification for the problem of deciding relevance, which has received significant attention in the classical abduction literature and mentioned as an open question.
    \end{enumerate}

\item We study the related problem of determining whether there exist two explanations of sufficiently high {\em diversity}. This metric can be precisely related to the existence of facets and several notions from our facet classification carry over. However, this problem is provably much harder and becomes NP-hard already for a small fragment of Horn consisting of just implication $(x \rightarrow y)$.
\end{enumerate}

\begin{figure}[htp]
    \centering
    \begin{subfigure}[t]{0.48\textwidth}
        \centering
        \includegraphics[width=\textwidth]{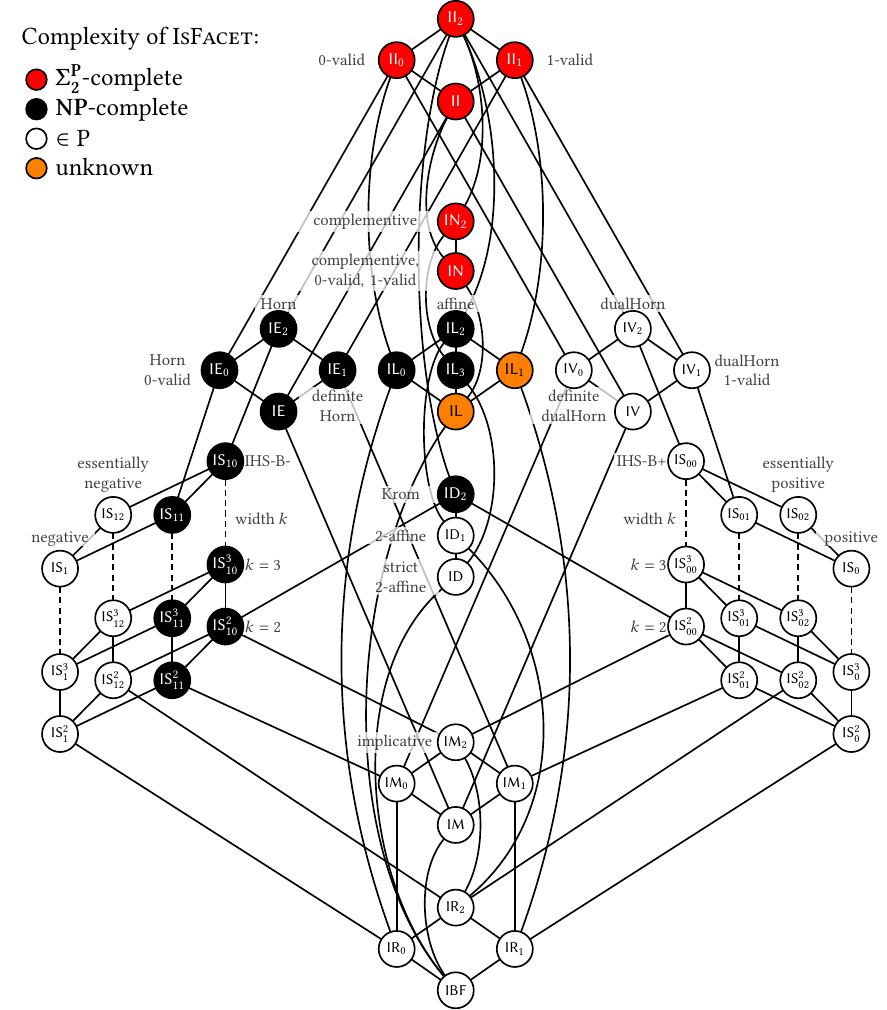}
        \caption{$\isfacet(\Gamma)$.}
        \label{fig:isfacet}
    \end{subfigure}
    \hfill
    \begin{subfigure}[t]{0.48\textwidth}
        \centering
        \includegraphics[width=\textwidth]{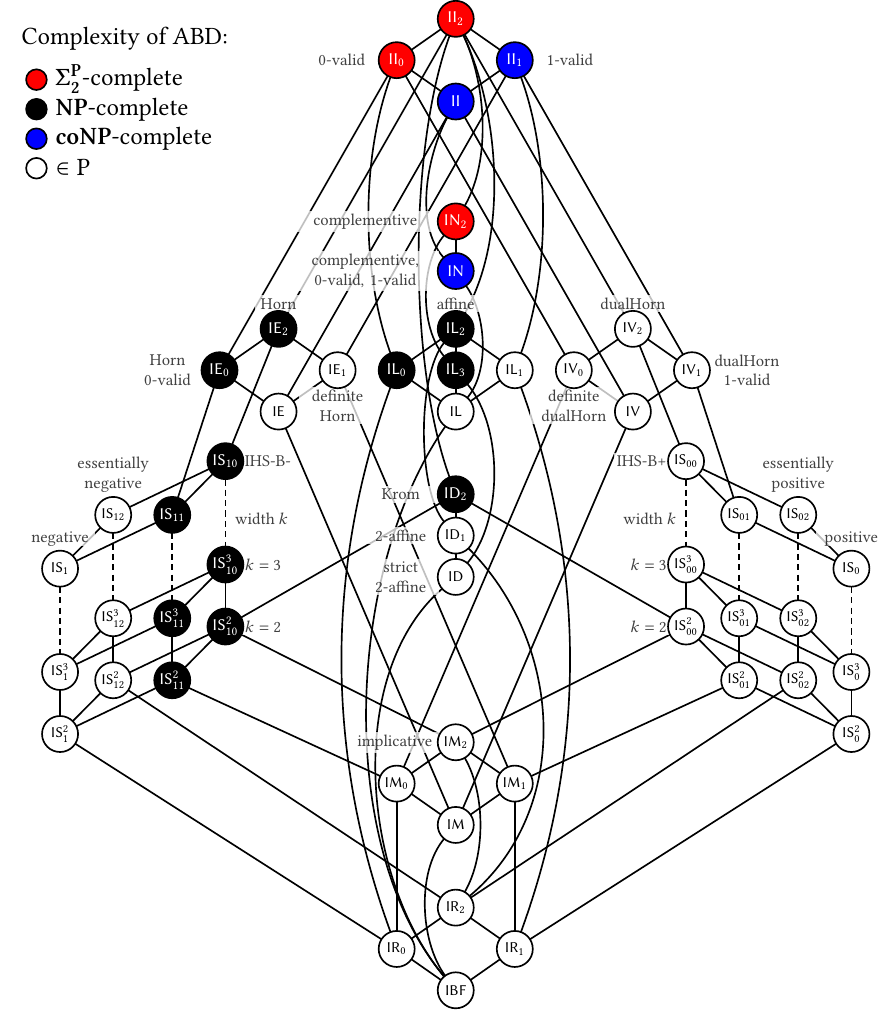}
        \caption{$\PMABD(\Gamma)$ \cite{NordhZ08}.}
        \label{fig:pabd}
    \end{subfigure}
    \caption{Illustration of complexity results via Post's lattice.}
    \label{fig:post-lattice}
\end{figure}
\paragraph{Related Works}
The computational complexity of propositional abduction is well studied. \citex{SelmanLevesque90} showed that explanation existence is tractable for Horn theories for special cases of propositional abduction, but NP-hard when checking for an explanation that contains a particular proposition. 
\citex{BylanderEtAl91} generalized the results %
to ``best'' explanations (most plausible combination of hypotheses that explains all the data) and present a tractable sub-class of abduction problems.
\citex{EiterGotlob95} proved $\Sigma^\Ptime_2$-completeness for propositional abduction. \citex{CreignouZanuttini06} established a precise trichotomy (\Ptime, \NP, $\Sigma^\Ptime_2$) in Schaefer's framework where inputs are restricted to generalized conjunctive normal form or subsets.
\citex{NordhZ08} lifted results to propositional knowledge bases and establish a tetrachotomy (\Ptime, \NP, \co-\NP, $\Sigma^\Ptime_2$). 
\citex{CreignouSchmidtThomas10} showed a complete complexity classification for all considerable sets of Boolean functions.
Relevance and dispensability are comparably speaking not as well understood. An early result by~\citex{FriedrichGottlobNejdl1990} show that it is NP-hard to determine relevance over (definite) Horn theories, and~\citex{EiterGotlob95} additionally prove that it is $\SigmaP$-hard to decide relevance if the knowledge base is an arbitrary propositional formula. \citex{zanuttini03} later asks if there is a simple relationship between deciding relevance and the complexity of the underlying abduction problem. We find it surprising that Zanuttini's question (later repeated by~\citex{NordhZ08}) still remains unanswered given that we by now have a complete understanding of the classical complexity of virtually all propositional abduction problems.
In addition, counting %
and 
enumeration complexity is also well studied~\cite{CreignouEtAl19,CreignouSchmidtThomas10,HermannPichler10}.
\citex{FellowsEtAl12}, \citex{PfandlerEtAl13}, and \citex{MahmoodEtAl21} included semantical structural restrictions (parameterized complexity) in the complexity study.
The concept of facets has originally been introduced %
in the context of answer-set
programming to enforce/forbid atoms in solutions and %
systematically investigate solutions
without counting or
enumeration~\cite{alrabbaa-et-al-rulemlpr2018,FichteGagglRusovac22}. The
complexity of ASP facets for tight, normal, and disjunctive programs
was established very recently~\cite{RusovacEtAl24}.
\citex{SpeckEtAl25} lifted facets to symbolic planning for reasoning faster on the plan space and \citex{FichteFrohlichHecher25} introduced facets to abstract argumentation.
\citex{EiterGeibinger23} studied justifications for the presence, or absence, of an atom in the context of answer-set programming including so-called contrastive explanations. They provide a basic complexity theoretical characterization.
Diversity has been considered in the literature on propositional satisfiability and logic programming,~e.g.,~\cite{MisraM024,BohlGagglRusovac23}. 
Abductive logic programming (ALP) combines logic programming with abductive reasoning, which then allows for generating hypotheses to explain observed facts or goals. \citex{EiterGottlobLeone95} studied the complexity of ALP regarding consistency, relevance, and necessity but focusing on normal and disjunctive programs and commonly used semantics (well-founded, stable).

\section{Preliminaries} \label{sec:preliminaries}

\newcommand{\BigO}[1]{\ensuremath{\mathcal{O}(#1)}}
\newcommand{\CCard}[1]{\ensuremath{||#1||}}
\newcommand{\preduction}{\leq^{\mathCommandFont{\Ptime}}_m}

We follow standard notions in 
computational complexity theory~\cite{Papadimitriou94,AroraBarak09}, 
propositional logic~\cite{kleine-lettmann-1999}, and propositional abduction~\cite{BylanderEtAl91}. Below, we briefly state
relevant %
notations.

\subsection{Computational Complexity}
Let $\Sigma$ and $\Sigma'$ be some finite alphabets. We call $I
\in \Sigma^*$ an \emph{instance} and $\CCard{I}$ denotes the size of~$I$.  
A \emph{decision problem} is some subset~$L\subseteq \Sigma^*$. %
Recall that \Ptime{} and \NP are the complexity classes of all
deterministically and non-deterministically polynomial-time solvable
decision problems~\cite{Cook71}. %
A polynomial-time many-to-one reduction~($\preduction$) 
from $L$ to $L'$ is a function $r : \Sigma^* \rightarrow {\Sigma'}^*$ such that for all $I \in \Sigma^*$ we have $I \in L$
if and only if $r(I) \in L'$ and $r$ are computable in time $\BigO{\CCard{I}\cdot c}$ for some constant~$c$. 
In other words, a polynomial-time many-to-one reduction transforms instances of the decision problem $L$ into instances of decision problem $L'$ in polynomial time.
We also need the Polynomial Hierarchy
(PH)~\cite{StockmeyerMeyer73,Stockmeyer76,Wrathall76}.
In particular, $\Delta^\Ptime_0 \eqdef \Pi^\Ptime_0 \eqdef
\Sigma^\Ptime_0 \eqdef \Ptime$ and $\Delta^\Ptime_{i+1} \eqdef
P^{\Sigma^p_{i}}$, $\Sigma^\Ptime_{i+1} \eqdef
\NP^{\Sigma^\Ptime_{i}}$, and $\Pi^\Ptime_{i+1} \eqdef
\text{co}\NP^{\Sigma^\Ptime_i}$ for $i>0$ where $C^{D}$ is the class~$C$ of
decision problems augmented by an oracle for some complete problem in
class $D$.

\paragraph{Propositional Logic}
A \emph{literal} is a variable $x$ or its negation $\neg x$. 
A \emph{clause} is a disjunction of literals, often represented as a set. A clause of arity 1, i.e., either $(x)$ or $(\neg x)$, is a {\em unit clause}.
We work in a general setting where atoms can be expressions of the form $R(x_1,  \ldots, x_r)$ for variables $x_1, \ldots, x_r$ and an $r$-ary relation $R \subseteq \{0,1\}^r$. A function $f \colon \{x_1, \ldots, x_r\} \to \{0,1\}$ is then said to satisfy an atom $R(x_1, \ldots, x_r)$ if $(f(x_1), \ldots, f(x_r)) \in R$. A (conjunctive) {\em propositional formula} $\varphi$ is a conjunction of atoms and we
write $\var(\varphi)$ for its set of variables.
A mapping $\sigma\colon \var(\varphi) \mapsto \{0,1\}$ is called an \emph{assignment} to the variables of~$\varphi$ and
a \emph{model} of a formula $\varphi$ is an assignment to $\var(\varphi)$ that satisfies $\varphi$.
For two formulas $\psi$ and $\varphi$, we write $\psi  \models \varphi$ if every model of $\psi$ also satisfies $\varphi$.

\subsection{Restrictions of Constraint Languages}
As alluded in Section~\ref{sec:intro}, we work in a fine-grained setting where not all possible relations are allowed. Formally, we say that a {\em constraint language} $\Gamma$ is a set of Boolean relations, and a {\em $\Gamma$-formula} is a propositional formula $\varphi$ where $R \in \Gamma$ for each atom $R(x_1, \ldots, x_r)$. For a constraint language $\Gamma$, we write $\SAT(\Gamma)$ for the problem of deciding if a given $\Gamma$-formula admits at least one model.
If $\Gamma$ is naturally expressible as a set of clauses, we represent $R \in \Gamma$ in clausal form. 
This is the the case for most, but not all, cases that we consider in this paper.
Usually, we do not %
distinguish between the relation, its defining clause, or an atom involving the clause. For example, we simply write $(x)$ for the unary relation $\{(1)\}$, $(\neg x)$ for $\{(0)\}$, $(x_1 \rightarrow x_2)$ or $(\neg x_1 \lor x_2)$ for $\{(0,0), (0,1), (1,1)\}$, and so on. The empty set $\emptyset$ is the (nullary) relation that is always false, and we write $(x_1 = x_2)$ for the equality relation $\{(0,0), (1,1)\}$.
\begin{table}
  \centering
  \rowcolors{2}{gray!25}{white}
  \resizebox{\linewidth}{!}{%
    \begin{tabular}{lll}
      \toprule
      Name & Definition & Corresponding co-clone \\
      \midrule
      $\text{CNF}$         & $\{c \mid c \text{ is a clause}\}$ & $\BR$ ($\II_2$) \\
      $\text{Horn}$        &  $\{c \mid c \text{ is a clause}, \text{Pos}(c) \leq 1\}$ & $\IE_2$ \\
      $\text{dualHorn}$    & $\{c \mid c \text{ is a clause}, \text{Neg}(c) \leq 1\}$ & $\IV_2$ \\
      $\text{EN}$          & $\{c \mid c \text{ is a clause}, \text{Pos}(c) = 0\} \cup \{(x), (x = y)\}$ & $\IS{}{12}$ \\
      $\text{EP}$          & $\{c \mid c \text{ is a clause}, \text{Neg}(c) = 0\} \cup \{(\neg x), (x = y)\}$ & $\IS{}{02}$ \\
      $\text{affine}$      & $(x_1 \oplus \ldots \oplus x_k) = b$, $k \geq 1$, $b \in \{0,1\}$ & $\IL_2$ \\
      \bottomrule
    \end{tabular}
  }
  \caption{%
  Constraint languages and their corresponding co-clones. Here, $\text{Pos}(c)$ and $\text{Neg}(c)$ denote the number of positive and negative literals in a clause $c$, respectively.
  For more details, we refer to the work by~\citex{BohlerRSV05} %
  and Table~\ref{tab:bases} in the supplemental material to this paper.
  }
  \label{tab:constraint-languages}
\end{table}
For a constraint language $\Gamma$ and $k \geq 1$, we often use the notation $k$-$\Gamma$ for the set of 
relations/clauses of arity at most $k$. Thus, 2-CNF contains all 1/2-clauses, and 2-affine contains the 
unary/binary relations definable as equations mod 2. Additionally, for a language $\Gamma$ we, let (1) 
$\Gamma^{\text{-}} = \Gamma \setminus \{(x), (\neg x)\}$ be $\Gamma$ without the two unit clauses, and (2) 
$\Gamma^+ = \Gamma \cup \{(x), (\neg x)\}$ be $\Gamma$ expanded with the two unit clauses. A language $\Gamma$ 
is {\em $b$-valid} for $b \in \{0,1\}$, if $(b, \ldots, b)\in R$ 
for 
each $R \in \Gamma$.
We introduce the most important constraint languages for 
the purpose of this paper in Table~\ref{tab:constraint-languages}.
This  only covers a small number of the possible 
constraint languages. For many applications, including the facet classification in this paper, doing an 
exhaustive case analysis of {\em all} possible constraint languages is too complicated, and one needs 
simplifying assumptions. Here, it is known that each 
constraint language $\Gamma$ can equivalently well  be 
described as a set of functions closed under functional 
composition and containing all {\em projections} $\pi^n_i(x_1, \ldots x_n) = x_i$, {\em clones}. 
Thus, each clone groups together many similar constraint 
languages and the Boolean clones form a lattice known as 
{\em Post's lattice} when ordered by set inclusion~\cite{Post41}. Many classification tasks become substantially simpler via 
Post's lattice, and it is well-known that each clone corresponds to a dual relational object called a {\em co-clone}, which in turn induces a useful closure property 
on relations. In this paper we only need a small fragment of this algebraic theory and define this closure property via so-called {\em primitive positive definitions} (pp-definitions) and say that an $r$-ary 
relation $R$ has a pp-definition over %
$\Gamma$ if \\[-1.5em]
$$R(x_1, \dots, x_r) := \exists y_1, \dots, y_n\, .\, \varphi(x_1, \ldots, x_r, y_1, \ldots, y_n)$$ 
\noindent
where $\varphi$ is a $(\Gamma \cup \{(x = y)\})$-formula. Thus, put otherwise, $R$ can be defined as the set of models of $\exists y_1, \ldots, y_n\, .\, \varphi(x_1, \ldots, x_r, y_1, \ldots, y_n)$ with respect to the free variables $x_1, \ldots, x_r$. The reason for allowing the equality relation $(x = y)$ as an atom is that it leads to a much simpler algebraic theory. However, we sometimes need the corresponding definability notion without (free) equality. A pp-definition where $\varphi$ is a $\Gamma$-formula is called an {\em equality-free primitive positive definition} (efpp-definition).

\begin{definition}
For a constraint language $\Gamma$, we let $\cclone{\Gamma}$ and  $\cclone{\Gamma}_{\neq}$, resp., be the smallest set of relations containing $\Gamma$ and where $R \in \Gamma$ for any (ef)pp-definable relation $R$ over $\Gamma$.
\end{definition}  
\noindent The set $\Gamma$ is in this context said to be a \emph{base} and it is known that all co-clones can be defined in this way. For details, we refer to the work by~\citex{BohlerRSV05} and Table~\ref{tab:bases} in the supplemental material to this paper. These notions generalize and unify many types of reductions and definability notions in the literature. 

\begin{example}
For example, consider the classical reduction from $4$-SAT to $3$-SAT by splitting a 4-clause $(\ell_1 \lor \ell_2 \lor \ell_3 \lor \ell_4)$ into two 3-clauses $(\ell_1 \lor \ell_2 \lor x)$ and $(\ell_3 \lor \ell_4 \lor \neg x)$ where $x$ is a fresh variable. This can be viewed as a pp-definition $(\ell_1 \lor \ell_2 \lor \ell_3 \lor \ell_4) \equiv \exists x \, . \, (\ell_1 \lor \ell_2 \lor x) \land (\ell_3 \lor \ell_4 \lor \neg x)$, and we conclude that $k\text{-CNF} \subseteq \cclone{3\text{-CNF}}$ for any $k \geq 1$. 
\end{example}

A table of all Boolean co-clones is available in the supplemental material as well as another more elaborate example.

\newcommand{\csl}{\protect\ensuremath{S}}
\newcommand{\AllExpl}{\ensuremath{\mathcal{E}}}
\newcommand{\MinExpl}{\ensuremath{\mathcal{E}_M}}
\newcommand{\REL}{\ensuremath{\mathcal{R}\text{el}\mathcal{E}}}
\newcommand{\NEC}{\ensuremath{\mathcal{N}\text{ec}\mathcal{E}}}
\subsection{Propositional Abduction}
Let $\Gamma$ be a constraint language, for example, a set of clauses.
An instance~$I$ of the \emph{positive propositional} abduction
problem over $\Gamma$, $\PMABD(\Gamma)$ for short, 
is a tuple~$I=(\KB, H, M)$ %
with $\KB$ being a $\Gamma$-formula over a finite set of Boolean variables called the \emph{knowledge base} (or \emph{theory}), $H \subseteq \var(\KB)$ called
\emph{hypotheses}, $M\subseteq \var(\KB)$ called
\emph{manifestations}. %
Since we have defined a $\Gamma$-formula as a conjunctive formula with atoms from $\Gamma$ we sometimes take the liberty of viewing the knowledge base as a set rather than as a formula.
A \emph{positive explanation}~$E$, \emph{explanation} for short, is a subset~$E\subseteq H$ such that (i)~$\KB \land E$ is satisfiable and (ii)~$\KB \land E \models M$.
An explanation~$E$ is \emph{(subset-)minimal} if no other set~$E' \subsetneq E$ is an explanation of~$I$.

The problem~$\PMABD(\Gamma)$ %
asks whether there is an explanation, which in the decision context is the same as asking whether there is a minimal explanation.
If $\Gamma$ is arbitrary, we omit $\Gamma$ from the problem and
write $\PMABD$. Note that the complexity of $\PMABD(\Gamma)$ is completely determined \cite{NordhZ08} and illustrated in Figure~\ref{fig:pabd}.

We write $\AllExpl(I)$ to refer to the set of all explanations and $\MinExpl(I)$ for the set of all subset-minimal explanations.
A variable~$x \in H$ is \emph{relevant} if $x$ belongs to some subset-minimal explanation~$E \in \MinExpl(I)$ and \emph{necessary} if $x$ belongs to all subset-minimal explanations~$E \in \MinExpl(I)$.
We abbreviate the sets of all relevant and necessary variables by~$\REL(I)$ and $\NEC(I)$, respectively.

    \newcommand{\lwed}{\textbf{\underline{w}ednesday}}
    \newcommand{\swed}{w}
    \newcommand{\lcalm}{\textbf{\underline{c}alm}}
    \newcommand{\scalm}{c}
    \newcommand{\lrace}{\textbf{\underline{n}o-race}}
    \newcommand{\srace}{n}
    \newcommand{\lstorm}{\textbf{\underline{s}torm}}
    \newcommand{\sstorm}{s}
    \newcommand{\lrain}{\textbf{\underline{r}aining}}
    \newcommand{\srain}{r}
     \begin{example}\label{ex:running}
        Consider our abduction example from %
        the introduction, which we slightly extend.
        Therefore, let the $\PMABD$ instance $I=(\KB, H,M)$
        consist of the knowledge base %
        \begin{samepage}
        \[\KB = \{ %
        \lwed \rightarrow \lrain,\;
        \lwed \land \lcalm \rightarrow \lrace,\; %
        \]\vspace{-2em}
        \[
        \lwed \land \lstorm \rightarrow \lrace\},%
        \]%
        the manifestation~$\srace$, and the hypothesis~$\{\swed, \scalm, \sstorm, \srain\}$. 
        \end{samepage}
        
        The set of all %
        explanations~$\AllExpl(I)$ consists of the explanations %
        $\{\swed, \scalm\}$, 
        $\{\swed, \scalm, \sstorm\}$, 
        $\{\swed, \scalm, \srain\}$, 
        $\{\swed, \sstorm\}$,
        $\{\swed, \sstorm, \scalm\}$,
        $\{\swed, \sstorm, \srain\}$, and
        $\{\swed, \sstorm, \scalm, \srain\}$. 
        The explanations $\{\swed, \scalm\}$ and $\{\swed, \sstorm\}$ are subset-minimal and hence constitute the set $\MinExpl(I)$.
        When considering the elements of $\MinExpl(I)$, we immediately observe that 
        the set of relevant propositions~$\REL(I)$ is formed of $\swed$, $\scalm$, and $\sstorm$. Whereas the only element in the necessary propositions~$\NEC(I)$ is $\swed$.
    \end{example}

\section{Facets in Explanations}
In non-monotonic reasoning, we commonly consider knowledge bases with multiple possible solutions, each leading to different conclusions. Central decision-based reasoning problems consider all possible solutions and consider how a variable relates to all solutions. When asking for brave (credulous) or cautious (skeptical) reasoning, we decide whether a variable belongs to one solution or all solutions respectively. The underlying idea is that brave reasoning allows for multiple potential conclusions from a knowledge base, i.e., a knowledge base may have uncertain outcomes. In contrast, skeptical reasoning requires a guaranteed outcome. This concept is also known in abductive reasoning with relevant and necessary explanations and has been considered in the literature,~e.g.,~\cite{EiterGotlob95,EiterGottlobLeone97, FriedrichGottlobNejdl1990,NordhZ08,zanuttini03}.
However, a detailed complexity classification in Post's lattice is open to date.
In this section, we consider reasoning between relevant and necessary explanations, so-called facets.
Intuitively, a variable $p$ is a facet if it is (i) part of {\em some} explanation (relevant), but (ii) not included in {\em every} explanation (dispensible). %
We start by defining facets formally. %

\begin{definition}[Facets]
    Let $I=(\KB, H, M)$ be an $\PMABD$ instance. A variable~$x \in H$ is a \emph{facet} in the instance~$I$ if  $x \in \REL(I) \setminus \NEC(I)$.
\end{definition}
\noindent Based on this definition, we define a computational problem %
whose task is to decide whether a given variable is a facet or not. 

\problemDef{$\isfacet(\Gamma)$}
{$I=(\KB, H, M, x)$ where $(\KB, H, M)$ is an $\PMABD(\Gamma)$ instance and $x \in H$}
{Is $x$ a facet in $I$?}

\noindent
    The following example illustrates facets in the sailing scenario.

    \begin{example}[Cont.]
        Consider our $\PMABD$ instance from Example~\ref{ex:running}.
        Since $\REL(I) = \{\lwed, \lcalm, \lstorm\}$ and $\NEC(I) = \{\lwed\}$, we observe that 
        $\scalm$ and $\sstorm$ are facets allowing for a variability in explanations whereas $\lwed$ occurs in all explanations and is thus not a facet.
        Note that without minimality in the relevance definition $\lrain$ would (contrary to intuition) qualify as relevant and also as a facet.
    \end{example}

The dispensability condition can be checked fairly easy in many cases. To this end, we test, given $(\KB, H, M, x)$, whether $(\KB, H \setminus \{x\}, M)$ admits an explanation or not. In particular, if $\PMABD(\Gamma)$ is in \Ptime, we can check this in polynomial time. Thus, the interesting computational aspect of $\isfacet(\Gamma)$ is to decide when {\em both} relevance and  dispensability can be decided without a major blow up in complexity.
The core parts of our proofs give an immediate classification for relevance as well.
Overall, we obtain an almost complete classification of $\isfacet(\Gamma)$.

\begin{theorem} \label{thm:main}
    The classification of $\isfacet(\Gamma)$ in Figure~\ref{fig:isfacet} is correct.
\end{theorem}

Before we prove Theorem~\ref{thm:main} systematically, we observe 
that for most complexity questions of $\isfacet(\Gamma)$ it is sufficient to consider the efpp-closure $\closneq{\Gamma}$ of $\Gamma$. 

\begin{lemma}\label{lem:baseIndneq}
    Let $\Gamma$ and $\Gamma'$ be two constraint languages. If $\Gamma' \subseteq \closneq{\Gamma}$, then
     $\isfacet(\Gamma') \preduction
     \isfacet(\Gamma)$.
\end{lemma}
\begin{proof}[Proof (Idea).]
We omit details, since the construction is exactly the same as in previous work~\cite[Lemma~22]{NordhZ08}, but the basic idea is simply to replace each relation by the set of constraints prescribed by the efpp-definition, and introducing fresh variables (kept outside the hypothesis) for any existentially quantified variables.
This exactly preserves the set of (minimal) explanations. 
\end{proof}

Lifting Lemma~\ref{lem:baseIndneq} to pp-definability does not appear to be possible in general: the classical 
trick when encountering an equality constraint $(x = y)$ is to %
identify the two variables throughout the instance. But consider, for example, an instance with knowledge base $\{(x = y), (x \rightarrow m), (y \rightarrow m)\}$, $H = \{x,y\}\}$, and $M = \{m\}$. Then, $x$ and $y$
are both facets since $\{x\}$ and $\{y\}$ are both minimal explanations, but if we identify $y$ with $x$ and remove the equality constraint, we 
obtain the instance~$\{x \rightarrow m\}$ where $x$ is {\em not} a facet, since there is only one minimal explanation~$\{x\}$. 
However, the loss  of the equality relation in the efpp-closure $\closneq{\Gamma}$ turns  out to be manageable. 
We explain, in the proof of Theorem~\ref{thm:main}, why our results in the next %
two sections extend to all co-clones.
To obtain the systematic cases, we %
require numerous lemmas, which we establish in the following.

\subsection{Computational Upper Bounds}

Recall from Figure~\ref{fig:pabd} that 
$\PMABD(\Gamma)$ is always either in (i)~$\Ptime$, (ii)~(co)NP, or (iii)~$\Sigma^P_2$. Hence, our first task is to identify the corresponding classes for the $\isfacet(\Gamma)$ problem. Ideally, one could hope that $\isfacet(\Gamma)$ can be solved without a large increase in complexity,~e.g., going from $\Ptime$ to being $\NP$-hard. %
We will see that this can often, but not always, be achieved. We begin by analyzing the simple language $\{x \rightarrow y\}$ where the only allowed constraint is an implication between two variables. From Figure~\ref{fig:pabd}, we know that $\PMABD(\{x \rightarrow y\})$ is in \Ptime %
and this can be extended to $\isfacet(\{x \rightarrow y\})$ via a more involved algorithm.

\begin{lemma}\label{lem:facetImpP}
    $\isfacet(\{x \rightarrow y\}) \in P$.
\end{lemma}
\begin{proof}
Let $(\KB, H, M, x)$ be an instance of $\isfacet(\{x \rightarrow y\})$,~i.e., %
$\KB$ %
only consists of implications.
Note that $\KB$ is 1-valid. %
Consequently, there is an explanation if and only if $H$ is an explanation. To see this, we observe that $H$ is always consistent with $\KB$. Thus, it has ``maximal'' entailment power. %
To explain a single $m \in M$, a single $h \in H$ is always sufficient.
For $m \in M$, we denote by $h(m) = \{ h \in H \mid \KB \land h \models m\}$, i.e., %
all hypotheses from $H$ that \emph{explain} $m$. %

Now, we observe that $h(m)$ can be computed in polynomial time, for each $m\in M$. Denote by $M_x \subseteq M$ the manifestations from $M$ that are explained by $x$ alone, that is, $M_x = \{ m \in M \mid \KB \land x \models m\}$. The set $M_x$ can also be computed in polynomial time by repeatedly checking whether $\KB \land x \models m$.
Since $M_x$ is explained by $x$, %
we make $x$ ``relevant'' by finding an $E \subseteq H\setminus \{x\}$ that avoids explaining at least one $m \in M_x$. A maximal candidate for this is $H \setminus h(m)$.
We can accomplish this as follows:

\begin{algorithmic}[1]
    \STATE $E \gets$ 'none'
    \FOR{$m \in M_x$}
        \IF{$\KB \land H \setminus h(m) \models M \setminus M_x$} \STATE $E \gets H \setminus h(m)$ \# \emph{candidate found} \ENDIF
    \ENDFOR
    \IF{$E =$ 'none'} \RETURN False \# \emph{$x$ can not be made relevant} \ENDIF
    \RETURN $\KB \land H \setminus \{x\} \models M$ \# \emph{is there an explanation without $x$?}
\end{algorithmic}

\noindent %
This runs in p-time, since entailment for $\SAT({x \rightarrow y, x, \bar x})$ is in p-time~\cite{Schaefer78}.
\end{proof}

We continue with $\text{dualHorn}$ where $\PMABD$ is also in P. Here, membership in P for $\isfacet(\text{dualHorn})$ is less obvious. Given that $\PMABD(\text{Horn})$ is NP-complete, we %
could %
suspect that checking for relevance and dispensability is  computationally more expensive. First, we need the following technical lemma, where we recall that $\text{dualHorn}^{\text{-}} = \text{dualHorn} \setminus \{x, \neg x\}$ is the set obtained from $\text{dualHorn}$ by removing the two unit clauses.

\begin{restatable}[$\star$\protect\footnote{We prove statements marked by~``$\star$'' in the supplemental material.%
}]{lemma}{unitclauses}
\label{lem:unitclauses}
    $\isfacet(\text{dualHorn}) \preduction \isfacet(\text{dualHorn}^{\text{-}})$.
\end{restatable}

Next, we show that %
the result for $\isfacet(\{x \rightarrow y\})$ can be extended to $\isfacet(\text{dualHorn}^{\text{-}})$, and, thus, also to $\isfacet(\text{dualHorn})$ via Lemma~\ref{lem:unitclauses}. 

\begin{restatable}[$\star$]{lemma}{ubdualHorn}
\label{lem:dualhorn}
    $\isfacet(\text{dualHorn}^{\text{-}}) \in \Ptime$.
\end{restatable}

Our second major tractability result concerns 2-$\text{affine}$, i.e., either unit clauses or relations definable by $(x \oplus y = 0)$ (equality) or $(x \oplus y = 1)$ (inequality).

\begin{restatable}{lemma}{widthtwoaffineP}
\label{lem:width2affineP}
    $\isfacet(\text{2-affine}) \in \Ptime$. 
\end{restatable}

\begin{proof}
Let $(\KB, H, M, x)$ be an instance of $\isfacet(\text{2-affine})$. We assume that each relation in $\KB$ is represented by precisely one linear equation of arity at most 2, see~\cite{CreignouOS11} and \cite{MahmoodEtAl21}.
First, if $\KB$ is not satisfiable we answer no. Second, we propagate all unit clauses as in Lemma~\ref{lem:unitclauses}.
Each remaining equation then expresses either equality or inequality between two variables. With the transitivity of the equality relation and the fact that in the Boolean case $a \neq b \neq c$ implies $a = c$, we can identify equivalence classes of variables such that each two classes are either independent or they must have contrary truth values.
We call a pair of dependent equivalence classes $(X, Y)$ a \emph{cluster}, i.e., $X$ and $Y$ must take contrary truth values.
Denote by $X_1, \dots, X_p$ the equivalence classes that contain variables from $M$ such that $X_i \cap M \neq \emptyset$. 
Denote by $Y_1, \dots, Y_p$ the equivalence classes such that for each $i$ the pair $(X_i, Y_i)$ represents a cluster.
We make the following stepwise observations:
(1) there is an explanation if and only if $H \cap X_i \neq \emptyset$ for every $1 \leq i \leq p$, (2) a subset-minimal explanation is constructed by taking exactly one representative from each $X_i$. %
(3) $x$ can be made relevant if $x \in X_i$, for some $i$.
(4) $x$ is a facet if additionally each  $X_i$ contains at least one representative different from $x$.
These checks can be done in polynomial time. We conclude that $\isfacet(\text{2-affine}) \in \Ptime$.
\end{proof}

We continue with the corresponding membership questions for complexity classes above~$\Ptime$. Here, we %
make a case distinction of whether the underlying satisfiability problem is in P, and in particular whether $\SAT(\Gamma^+)$ is in P. We begin with the following %
lemma.

\begin{restatable}[$\star$]{lemma}{NPverify}\label{lem:NPverify}
    If $\SAT(\Gamma^+) \in \Ptime$, then there is a polynomial time algorithm to determine whether a given $E \subseteq H$ is an explanation for a given abduction instance $(\KB, H, M)$.
\end{restatable}

In particular, we obtain the following general statement, which shows that the complexity of $\isfacet(\Gamma)$ for many natural cases does not jump to $\SigmaP$.

\begin{restatable}[$\star$]{lemma}{InNp} \label{lemma:in_np}
    For any constraint language $\Gamma$, $\SAT(\Gamma^+) \in \Ptime \Rightarrow \isfacet(\Gamma) \in \NP$.
\end{restatable}

This covers a substantial number of cases since $\SAT(\Gamma^+)$ is in P when $\Gamma$ is {\em Schaefer},~i.e., contained in $\IV_2$ ($\text{dualHorn}$), $\IE_2$ ($\text{Horn}$), $\IL_2$ ($\text{affine}$), or $\ID_2$ ($\text{2-CNF}$). Our last major tractable case concerns the set of essentially negative clauses $\text{EN}$.

\begin{restatable}{lemma}{essentiallyNegative}
\label{lem:essentially_negative}
    $\isfacet(\text{EN}) \in P$.
\end{restatable}

\begin{proof}
We assume an arbitrary instance of $\isfacet(EN)$: $(\KB,H,M,x)$. We first apply unit propagation, exactly as in Lemma~\ref{lem:unitclauses}. We can now assume that the instance only contains negative clauses of 
size $\geq 2$ and equality clauses. 
We organize all variables which are equal to each other into equivalence classes as in Lemma~\ref{lem:width2affineP}, with the exception that all 
classes are independent in this case. If some equivalence class that 
contains an $m_i \in M$  does not contain variables from $H$,  this 
$m_i$ cannot be entailed. Thus, the abduction problem has no solutions and no facets.
    
 Otherwise, if all classes that contain an $m_i \in M$ contain at least one variable from $H$, we set all variables in these classes to true 
 and check if this is consistent with $\KB$. This, guarantees the existence of abduction solutions. 
 If this is the case, we can check if %
 $x$ is a facet. For this, we need two conditions: let $x \in C$ where $C$ is an equivalence 
 class. First, there must be at least one manifestation $m_i \in C$, 
 else $x$ cannot imply any $m_i$ and thus is never needed in a subset-minimal solution ($x$ would not be relevant). Second, there must be 
 at least one variable $x_i\in C$ different from $x$, otherwise $x$ will 
 always be needed to explain $m_i$ and there can be no solution without it ($x$ would be necessary).
\end{proof}

\subsection{Computational Lower Bounds}
We begin with a general result that implies that the facet problem is always at least as hard as the underlying abduction problem, provided the Boolean equality  relation can be expressed.

\begin{lemma}\label{lem:abdIsfacet}
    $\PMABD(\Gamma) \preduction \isfacet(\Gamma)$ if $(x = y) \in \Gamma$ for any constraint language $\Gamma$.
\end{lemma}

\begin{proof}
Given an instance $(\KB, H, M)$ of $\PMABD(\Gamma)$ we let $x,y,m$ be fresh variables. We define the instance $(\KB', H', M', x)$ of $\isfacet(\Gamma)$ as $\KB' = \KB \cup \{(x=m), (y=m)\},\ H' = H \cup \{x,y\},\ M' = M \cup \{m\}$.
 We claim that $(\KB, H, M)$ admits an explanation if and only if $x$ is a facet in $(\KB', H', M')$. (``$\Rightarrow$''): %
 assume that $E \subseteq H$ is a subset-minimal explanation. Then, $E \cup \{x\}$ and $E \cup \{y\}$ are both subset-minimal explanations in $(\KB', H', M')$, so $x$ is a facet. 
 (``$\Leftarrow$''): assume that $x$ is a facet in $(\KB', H', M')$. Then, there exists a subset minimal explanation $E' \subseteq H'$ where $x \in E'$, and it follows that $E' \setminus \{x\}$ is a (subset-minimal) explanation for $(\KB, H, M)$.
\end{proof}

\noindent By combining this with Lemmas~\ref{lem:baseIndneq} and \ref{lem:express_equality}, %
we inherit all hardness results from $\PMABD$. All non-polynomial cases of $\PMABD(\Gamma)$ satisfy 
$\Gamma \not \subseteq \IS{}{12}$ and $\Gamma \not \subseteq \IS{}{02}$,~i.e., %
such languages are not essentially negative and not essentially positive~\cite{NordhZ08}.

\begin{lemma}{\cite[Lemma 9]{MahmoodEtAl21}}\label{lem:express_equality}
Let $\Gamma$ be a constraint language.
If $\Gamma \not \subseteq \IS{}{12}$ and $\Gamma \not \subseteq \IS{}{02}$, then $(x=y) \in \closneq{\Gamma}$ and $\clos{\Gamma} = \closneq{\Gamma}$.
\end{lemma}

However, the facet problem $\isfacet$ is generally even harder than $\PMABD$. We present a technical lemma, providing us the unit clause $(x)$ for free. Then, we will present languages where $\PMABD$ is polynomial and $\isfacet$ is NP-hard, as well as languages where $\PMABD$ is coNP-complete, while $\isfacet$ is $\Sigma_2^P$-complete.

\begin{restatable}[$\star$]{lemma}{lemFacetT}\label{lem:FacetT}
    For any constraint language $\Gamma$, it holds that $\isfacet(\Gamma \cup \{(x)\}) \preduction \isfacet(\Gamma)$.
\end{restatable}

We are now ready to state a crucial reduction, which is at the heart of the increased complexity of $\isfacet$ vs.\ $\PMABD$. The $\isfacet$-problem allows to simulate negative unit clauses, provided that the language is 1-valid and can express implication $(x \rightarrow y)$.

\begin{restatable}{lemma}{FacetF}
\label{lem:FacetF}
    If $\Gamma$ is 1-valid, then $\PMABD(\Gamma \cup \{(\neg x)\}) \preduction \isfacet(\Gamma \cup \{x \rightarrow y\})$.
\end{restatable}

\begin{proof}
     Let $(\KB, H, M)$ be an instance of $\PMABD(\Gamma \cup \{(\neg x)\})$. If $\KB$ contains two unit clauses $\neg x$ and $\neg y$ for distinct variables $x$ and $y$, %
     we may simply identify $x$ with $y$ and obtain an equivalent instance.
     Thus, we may %
     wlog assume that $\KB = \varphi \land (\neg z)$, $z \in \var(\varphi)$, for a $\Gamma$-formula $\varphi$. Let $x,y,m$ denote fresh variables and define $V = \var(\varphi) \cup H \cup M \cup \{x,y,m\}$. We define the target instance $(\KB', H', M', x)$ as
    \[\KB' = \varphi \land \bigwedge_{x_i \in V} (z \rightarrow x_i) \land (x \rightarrow m) \land (y \rightarrow m),\quad H' = H \cup \{x,y\},\quad M' = M \cup \{m\}.\]

Note that $\KB'$ is a $\Gamma \cup \{x \rightarrow y\}$-formula, as required.

In the following, we prove correctness formally.
Observe first that for $z=0$, $\KB$ and $\varphi$ have exactly the same models (upto the fresh variables $x,y,m$). For $z=1$, $\varphi$ may admit additional models. However, due to $\KB'$ containing the construct $\bigwedge_{x_i \in V} (z \rightarrow x_i)$, the only additional model is the all-1 model.

\noindent\textbf{Correctness:}\\
\noindent
(``$\Rightarrow$''): Be $E\subseteq H$ an explanation for $(\KB, H,M)$. Then, with the above observation that the only additional model is the all-1 model (which satisfies $M$ and $m$), it is easily observed that
\begin{enumerate}
    \item $E' = E \cup \{x\}$ constitutes an explanation for $(\KB', H', M')$
    \item $E' \setminus \{x\}$ is \textit{no} explanation for $(\KB', H', M')$
    \item there is an explanation without $x$, namely $E \cup \{y\}$
\end{enumerate}
In summary, $x$ is a facet.

\medskip

\noindent
(``$\Leftarrow$''): Be $x$ a facet for $(\KB', H', M')$. Then there is a set $E' \subseteq H'$ such that
\begin{enumerate}
   \item $E'$ is an explanation for $(\KB', H', M')$
   \item $E'\setminus \{x\}$ is \textit{no} explanation for $(\KB', H', M')$
   \item there is an explanation for $(\KB', H', M')$ without $x$
\end{enumerate}
From the construction it is easily observed that $E'$ must be of the form $E' = E \cup \{x\}$, for an $E \subseteq H$. Since $E'\setminus \{x\} = E$ fails as explanation for $(\KB', H', M')$, and $E$ is obviously consistent with (1-valid) $\KB'$, we conclude that $E$ fails due to $\KB' \land E$ not entailing $M' = M \cup \{m\}$. That is,
\begin{align}\label{eq:no1}
\KB' \land E \not\models M \cup \{m\}
\end{align}
\noindent
Further, since $E \cup \{x\}$ is an explanation, we know that $\KB' \land E \cup \{x\}$ \emph{does} entail $M' = M \cup \{x\}$. Since by construction, $x$ can not be responsible for entailing $M$, we conclude that $\KB' \land E$ entails $M$.
That is,
\begin{align}\label{eq:no2}
\KB' \land E \models M
\end{align}
\noindent
From (\ref{eq:no1}) and (\ref{eq:no2}) we conclude that 
that $\KB' \land E \not \models m$. From this we conclude that $\KB' \land E$ admits models where $z = 0$ (otherwise,
$\KB' \land E$ would entail $m$, due to $\KB'$ containing $z \rightarrow m$).
Therefore, we conclude that $\KB'[z=0] \land E$ admits models (is consistent) and entails $M$. Since $\KB'[z=0] \equiv \KB$ (upto the ``irrelevant'' variables $x,y,m$) we conclude that 1) $\KB \land E$ is consistent, and 2) $\KB \land E \models M$. That is, $E$ is an explanation for $(\KB, H, M)$.
\end{proof}

We are now ready to derive the hardness results in a series of short, technical lemmas.

\begin{restatable}[$\star$]{lemma}{FacetFext}\label{lem:FacetFext}
    If $\II_1 = \clos{\Gamma}$ or $\IE_1 = \clos{\Gamma}$, then 
    $\PMABD(\Gamma \cup \{(\neg x)\}) \preduction \isfacet(\Gamma)$.
\end{restatable}

\begin{restatable}[$\star$]{lemma}{FacetFextB}\label{lem:FacetFextB}
    If $\IN \subseteq \clos{\Gamma}$, then 
    $\isfacet(\Gamma)$ is $\SigmaP$-hard.
\end{restatable}

\begin{restatable}[$\star$]{lemma}{iehard}\label{lem:ie_hard}
    If $\IE \subseteq \clos{\Gamma}$, then 
    $\isfacet(\Gamma)$ is $\NP$-hard.
\end{restatable}

We remark that \citex{FriedrichGottlobNejdl1990} prove NP-hardness for the relevance problem for $\IE_1$. While this proof can be adapted to our setting it does {\em not} generalize to the other cases in this section.
By combining all results, %
we obtain the main result of the paper.

\begin{proof}[\textbf{Proof of Theorem~\ref{thm:main}}]
First, we observe that the each language $\Gamma$ considered in Lemma~\ref{lem:dualhorn}, Lemma~\ref{lem:width2affineP}, or Lemma~\ref{lem:essentially_negative} either contains or can define equality $(x = y)$. Hence, Lemma~\ref{lem:baseIndneq} is applicable and proves tractability for any language $\Delta$ such that $\Delta \subseteq \closneq{\Gamma} = \cclone{\Gamma}$. This covers all tractable cases in Figure~\ref{fig:isfacet}.

For intractability, all NP-complete cases {\em except} $\IE_1$ and $\IE$ follow from Lemma~\ref{lemma:in_np} (since $\SAT(\Gamma^+)$ is in P for every such $\Gamma$), Lemma~\ref{lem:abdIsfacet}, and Lemma~\ref{lem:express_equality}. The former two cases are instead proven to be NP-hard in Lemma~\ref{lem:ie_hard}, and inclusion in NP follows from Lemma~\ref{lemma:in_np}. Last, $\SigmaP$-hardness for all remaining cases follow from Lemma~\ref{lem:FacetFextB}, and inclusion in $\SigmaP$ is straightforward via arguments similar to Lemma~\ref{lemma:in_np}.
\end{proof}

We view the two missing cases $\IL$ (even linear equations) and $\IL_1$ (even linear equations, and unit clauses) as interesting future research questions. However, via Lemma~\ref{lemma:in_np} we may at least observe that $\isfacet(\Gamma) \in \NP$ for any base $\Gamma$ of $\IL$ or $\IL_1$. Hence, the only question remaining is whether these problems are in P, NP-complete, or --- unlikely but still possible --- NP-intermediate.

We conclude this section with the observation that all membership and hardness proofs, which we have given for the $\isfacet$ problem are also applicable to the relevance problem (deciding if $x \in H$ is relevant).

\begin{theorem}
    The classification of $\isfacet(\Gamma)$ in Figure~\ref{fig:isfacet} also describes the complexity of the relevance problem.
\end{theorem}

\begin{proof}
    Recall that an instance of the relevance problem is given by $ I = (\KB, H, M, x)$ and the question is whether $x \in \REL(I)$, that is, whether $x$ belongs to a subset-minimal explanation.
    We revisit now all membership and hardness proofs for $\isfacet$ and observe that they can easily be adapted to the relevance problem.

    \emph{Membership in $\Ptime$}. First observe that the reduction of Lemma~\ref{lem:unitclauses} to get rid of unit clauses can be performed analogously on the relevance problem.
    Next we observe that all algorithms showing P-membership, that is,
    Lemmas~\ref{lem:facetImpP},
    \ref{lem:dualhorn},
    \ref{lem:width2affineP}, and
    \ref{lem:essentially_negative}, first decide whether the given $x$ is relevant, and then in a second (independent) step decide whether $x$ is dispensable (not necessary).
    By dropping the dispensability check, we obtain a polynomial time algorithm to decide the relevance problem.

    \emph{Membership in $\NP$}. Analogously to the P-membership algorithms we drop the step of the dispensability check: in Lemma~\ref{lemma:in_np} omit guessing an $E_2 \subseteq H \setminus \{x\}$ and verifying that $E_2$ is an explanation.

    \emph{Membership in $\SigmaP$}. Inclusion in $\SigmaP$ is straightforward via arguments similar to Lemma~\ref{lemma:in_np} using an NP-oracle.

    \emph{$\NP$-hardness and $\SigmaP$-hardness}. First observe that Lemma~\ref{lem:baseIndneq} is also applicable to relevance, since the underlying reduction preserves the exact set of explanations. Next observe that the reduction from Lemma~\ref{lem:abdIsfacet} carries over one-to-one to relevance. It is optional to simplify the proof by removing the clause $(y = m)$ and the variable $y$ (whose only purpose was to assure that $x$ is not necessary). Lemma~\ref{lem:FacetT} is easily observed to carry over one-to-one to relevance. Lemma~\ref{lem:FacetF} carries over one-to-one, again it is optional to simplify the proof by removing $y$ and the clause $(y \rightarrow m)$. Finally, Lemmas~\ref{lem:FacetFext},
    \ref{lem:FacetFextB}, and
    \ref{lem:ie_hard} hold analogously, since all used lemmas carry over, as shown above.

    In summary, this theorem is proven analogously to Theorem~\ref{thm:main}.
\end{proof}

\section{Diverse Explanations}
In the previous section on facets, we considered whether there 
exists a variable that is relevant in explanations but dispensable. 
Thereby, we obtain a notion on flexibility on one variable belonging 
to explanations. Now, we lift flexibility from one variable to a set of 
variables in explanations and ask whether 
there exist two explanations of sufficiently high {\em diversity}. 
It  turns out that this metric can be precisely related to the existence of facets and several notions from our facet classification carry over. 
However, measuring the distance is provably much harder and becomes NP-hard already for the small fragment of Horn consisting of a single implication $(x \rightarrow y)$.
Before, we illustrate our complexity results in detail, we define our distance measure and computational problem.
\begin{definition}
Let $I=(\KB, H, M)$ be an $\PMABD$ instance, and $E_1\subseteq H$ and $E_2\subseteq H$ be two sets of variables over the hypotheses~$H$. 
Then, the \emph{distance} $d(E_1, E_2)$ between $E_1$ and $E_2$ is the cardinality of their symmetric difference. More formally, 
\begin{align*}
d(E_1, E_2) \eqdef &\, \Card{E_1 \triangle E_2} = \Card{(E_1 \cup E_2) \setminus (E_1 \cap E_2)} \\
=&\, \Card{\{x \in H \mid x \in E_1 \text{ and } x \notin E_2 \text{ or } x \notin E_1 \text{ and } x \in E_2\}}.
\end{align*}
If $E_1$ and $E_2$ are explanations, i.e., $E_1, E_2 \in \AllExpl(I)$, and $d(E_1,E_2) \geq k$, then we call $E_1$ and $E_2$ {\em $k$-diverse explanations}.
\end{definition}

\noindent
Note that the maximum distance is $|H|$, which is reached by $d(H, \emptyset)$. Our notion of distance is in line with the corresponding notion for $\SAT(\Gamma)$~\cite{MisraM024} and many other diversity problems studied in AI.
Note that we do {\em not} require that the two explanations are minimal, since the distance notion does not require this.

Next, we define the \emph{diversity problem for abduction}. 

\problemDef{$\divABD(\Gamma)$}
{An $\PMABD(\Gamma)$ instance $I=(\KB, H, M)$ and $k \geq 0$}
{Does $I$ have two $k$-diverse explanations $E_1$ and $E_2$?}

We have the following relationship to facets.

\begin{prop}\label{prop:facetdiv}
Let $I = (\KB, H, M)$ be an instance of $\PMABD(\Gamma)$ and  $E_1, E_2 \in \MinExpl(I)$. Then, every $x \in E_1 \triangle E_2$ is a facet.
\end{prop}

\begin{proof}
    Let $x \in E_1 \triangle E_2 = (E_1 \cup E_2) \setminus (E_1 \cap E_2)$. We observe that if $x$ is {\em not} a facet then either (1) $x \in E_1 \cap E_2$ since it is part of {\em every} explanation, or (2) $x \notin E_1 \cup E_2$ since it is not included in {\em any} subset-minimal explanation. Hence, $x \notin E_1 \triangle E_2$, meaning that $x \in  E_1 \triangle E_2$ is only possible if $x$ is a facet.
\end{proof}

From the relationship between facets and the distance notion, which we 
establish in Proposition~\ref{prop:facetdiv}, we can suspect 
similarities between the computational problems $\isfacet(\Gamma)$ and 
$\divABD(\Gamma)$. First, we establish that all hardness results are inherited from $\PMABD$,  analogously to 
Lemma~\ref{lem:abdIsfacet} and Lemma~\ref{lem:baseIndneq} for $\isfacet(\Gamma)$.

\begin{restatable}[$\star$]{lemma}{LemDivEq}
  $\PMABD(\Gamma) \preduction \divABD(\Gamma)$ if $(x = y) \in \Gamma$.
\end{restatable}

Also, $\SigmaP$-hardness for 1-valid and complementive languages is obtained analogously.

  \begin{restatable}[$\star$]{lemma}{LemDivIN}
    If $\IN \subseteq \clos{\Gamma}$, then 
    $\divABD(\Gamma)$ is $\SigmaP$-hard.
  \end{restatable}

However, the problem $\divABD$ is generally harder than $\isfacet$. 
Below, we establish \NP-hardness for simple implicative languages.  We
reduce from the problem $\divPosTwoSAT$ where we are given a positive
2-CNF formula $\varphi$ and an integer $k$.  Therefore, we
require %
two models of Hamming distance at least~$k$, which
is \NP-hard~\cite{MisraM024}.
\begin{restatable}[$\star$]{lemma}{LemDivImp}
    $\divPosTwoSAT \preduction \divABD(\{x \rightarrow y\})$.
\end{restatable}

Thus, despite the similarities between $\isfacet(\Gamma)$ and $\divABD(\Gamma)$, the latter seems to be significantly harder. However, we observe two tractable cases.

\begin{restatable}[$\star$]{lemma}{lbAffine}\label{lem:lbAffine}
    $\divABD(\text{2-affine}), \divABD(\text{EP})  \in \Ptime$. 
\end{restatable}

\section{Conclusion}

In this paper, we introduce faceted reasoning to propositional abduction. We illustrate that this reasoning
allows more fine-grained decisions 
than previously explored notions such as relevance and 
necessity/dispensability. 
We relate facets to the problem of finding diverse explanations. 
We establish an almost complete complexity classification in Post's lattice. 
In many cases facets can be found without a major blow-up in complexity. 
This is particularly interesting, given that {\em counting} minimal
explanations is almost always substantially harder~\cite{HermannPichler10}. 
Our facet classification also implies a corresponding classification for the relevance problem, thus answering an open question~\cite{NordhZ08,zanuttini03}.  %
For diversity, our results are less conclusive, but %
any major tractable cases seem unlikely, since it is hard already for the fragment $(x \rightarrow y)$.

\paragraph{Completing the trichotomy} 
The two open cases, affine equations of odd length with, or without, unit clauses, could be interesting to resolve. It seems unlikely that $\isfacet(\Gamma)$ could be tractable for such languages, but, at the same time, any hardness proof likely needs to involve significant new ideas. Note %
that affine languages were also absent in earlier counting complexity classifications~\cite{HermannPichler10}. Hence, there %
is %
a blind spot for complexity of abduction. One possible way forward could be to first classify the complexity of $\divABD(\Gamma)$ for all affine languages, which seems likely to be hard. 

\smallskip\noindent \textit{Parameterized complexity and diversity.}
It is reasonable that $\divABD(\Gamma)$ can be fully classified, but without any large, tractable cases.
However, %
more fine-grained investigations involving problem structure (parameter) remains interesting.
For $\divABD(\Gamma)$ a %
natural parameter is the maximum allowed distance. A systematic %
classification could not only open up new efficiently solvable cases but could also prove to be a useful framework for proving hardness for other types of diversity problems, especially, since the classical complexity of abduction is much richer than the one for SAT.

\paragraph{Abductive Logic Programming (ALP)}
Since detailed complexity results on logic programs~\cite{Truszczyski11} and results on  
the complexity of ALP regarding consistency, relevance, and necessity~\cite{EiterGottlobLeone95} exist, it could be interesting to extend our results to stable models in %
ALP where the input is given in form of rules. %

\paragraph{Applications}
Faceted reasoning can aid the search for heterogeneous explanations, which could be valuable in any domain with classical applications of abduction, e.g., diagnosis, %
and explainable AI.
More concretely, a practical application of facets are logistics applications where solutions need to be explained such as in the Beluga AI Competition~\cite{GnadHecherGaggl25}. %
There, several tasks aim at flexibility in explanations or alternatives.
Finally, we hope that our results on diverse explanations will spark interest into a deeper complexity study and exploration to use these for tasks like model debugging or decision support.
Beyond abductive reasoning, we expect that facets and diversity could
be interesting for epistemic logic
programs~\cite{EiterFichteHecher24}, default
logic~\cite{FichteMeierSchindler24}, and probabilistic
reasoning~\cite{FichteHecherNadeem22}.

\section*{Acknowledgements}
The work has received funding from the Swedish research council under
grant VR-2022-03214 and ELLIIT funded by the Swedish government.

\bibliographystyle{tlplike}

\label{lastpage}

\clearpage

\appendix
\section{Omitted Definitions}

See Table~\ref{tab:bases} for a comprehensive list of all Boolean co-clones.

\begin{table*}
  \centering
  \rowcolors{2}{gray!25}{white}
  \resizebox{\linewidth}{!}{%
    \begin{tabular}{llll}\toprule
      co-clone        & base                                     & clauses/equation                                                                                            & name/indication                             \\\midrule
      $\BR$ ($\II_2$) & 1-IN-3 = $\{001, 010, 100\}$             & all clauses                                                                                            & all Boolean relations                       \\
      $\II_1$         & $x \lor (y \oplus z)$                    & at least one positive literal per clause                                                               & 1-valid                                     \\
      $\II_0$         & DUP, $x \rightarrow y$                   & at least one negative literal per clause                                                               & 0-valid                                     \\
      $\II$           & EVEN$^4$, $x \rightarrow y$              & at least one negative and one positive literal per clause                                              & 1- and 0-valid                              \\
      $\IN_2$         & NAE = $\{0,1\}^3 \setminus \{000,111\}$  & cf. previous column                                                                                    & complementive                               \\ %
      $\IN$           & DUP = $\{0,1\}^3 \setminus \{101, 010\}$ & cf. previous column                                                                                    & complementive and 1- and 0-valid            \\
      $\IE_2$         & $x \land y \rightarrow z, x, \neg x$     & clauses with at most one positive literal                                                              & Horn                                        \\
      $\IE_1$         & $x \land y \rightarrow z, x$             & clauses with exactly one positive literal                                                              & definite Horn                               \\
      $\IE_0$         & $x \land y \rightarrow z, \neg x$        & $(x_1 \lor \neg x_2 \lor \dots \lor \neg x_n), n\geq 2, (\neg x_1 \lor \dots \lor \neg x_n), n \geq 1$ & Horn and 0-valid                            \\
      $\IE$           & $x \land y \rightarrow z$                & $(x_1 \lor \neg x_2 \lor \dots \lor \neg x_n), n\geq 2$                                                & Horn and 1- and 0-valid                     \\
      $\IV_2$         & $x \lor y \lor \neg z, x, \neg x$        & clauses with at most one negative literal                                                              & dualHorn                                    \\
      $\IV_1$         & $x \lor y \lor \neg z, x$                & $(\neg x_1 \lor x_2 \lor \dots \lor x_n), n\geq 2, (x_1 \lor \dots \lor x_n), n \geq 1$                & dualHorn and 1-valid                        \\
      $\IV_0$         & $x \lor y \lor \neg z, \neg x$           & clauses with exactly one negative literal                                                              & definite dualHorn                           \\
      $\IV$           & $x \lor y \lor \neg z$                   & $(\neg x_1 \lor x_2 \lor \dots \lor x_n), n\geq 2$                                                     & dualHorn and 1- and 0-valid                 \\
$\IL_2$               & EVEN$^4$, $x$, $\neg x$                  & all affine clauses (all linear equations)                                                              & affine                                      \\
$\IL_1$               & EVEN$^4$, $x$                            & $(x_1 \oplus \dots \oplus x_n = a)$, $n\geq 0, a = n$ (mod 2)                                          & affine and 1-valid                          \\
$\IL_0$               & EVEN$^4$, $\neg x$                       & $(x_1 \oplus \dots \oplus x_n = 0)$, $n\geq 0$                                                         & affine and 0-valid                          \\
$\IL_3$               & EVEN$^4$, $x \oplus y$                   & $(x_1 \oplus \dots \oplus x_n = a)$, $n$ even, $a \in \{0,1\}$                                         & -                                           \\
$\IL$                 & EVEN$^4$                                 & $(x_1 \oplus \dots \oplus x_n = 0)$, $n$ even                                                          & affine and 1- and 0-valid                   \\
$\ID_2$               & $x \oplus y, x \rightarrow y$            & clauses of size 1 or 2                                                                                 & Krom, bijunctive, 2-CNF                      \\
$\ID_1$               & $x \oplus y, x, \neg x$                  & affine clauses of size 1 or 2                                                                          & 2-affine                                    \\
$\ID$                 & $x \oplus y$                             & affine clauses of size 2                                                                               & strict 2-affine                             \\
$\IM_2$               & $x \rightarrow y,  x, \neg x$            & $(x_1 \rightarrow x_2), (x_1), (\neg x_1)$                                                             & implicative                                 \\
$\IM_1$               & $x \rightarrow y, x$                     & $(x_1 \rightarrow x_2), (x_1)$                                                                         & implicative and 1-valid                     \\
$\IM_0$               & $x \rightarrow y, \neg x$                & $(x_1 \rightarrow x_2), (\neg x_1)$                                                                    & implicative and 0-valid                     \\
$\IM$                 & $x \rightarrow y$                        & $(x_1 \rightarrow x_2)$                                                                                & implicative and 1- and 0-valid              \\
$\IS{}{10}$           & cf. next column                          & $(x_1), (x_1 \rightarrow x_2), (\neg x_1 \lor \dots \lor \neg x_n), n \geq 0$                          & IHS-B-                                      \\
$\IS{k}{10}$          & cf. next column                          & $(x_1), (x_1 \rightarrow x_2), (\neg x_1 \lor \dots \lor \neg x_n), k \geq n \geq 0$                   & IHS-B- of width $k$                         \\
$\IS{}{12}$           & cf. next column                          & $(x_1), (\neg x_1 \lor \dots \lor \neg x_n), n \geq 0, (x_1 = x_2)$                                    & essentially negative ($\eneg$)              \\
$\IS{k}{12}$          & cf. next column                          & $(x_1), (\neg x_1 \lor \dots \lor \neg x_n), k \geq n \geq 0, (x_1 = x_2)$                             & essentially negative ($\eneg$) of width $k$ \\
$\IS{}{11}$           & cf. next column                          & $(x_1 \rightarrow x_2), (\neg x_1 \lor \dots \lor \neg x_n), n \geq 0$                                 & -                                           \\
$\IS{k}{11}$          & cf. next column                          & $(x_1 \rightarrow x_2), (\neg x_1 \lor \dots \lor \neg x_n), k \geq n \geq 0$                          & -                                           \\
$\IS{}{1}$            & cf. next column                          & $(\neg x_1 \lor \dots \lor \neg x_n), n \geq 0, (x_1 = x_2)$                                           & negative ($\negg$)                          \\
$\IS{k}{1}$           & cf. next column                          & $(\neg x_1 \lor \dots \lor \neg x_n), k \geq n \geq 0, (x_1 = x_2)$                                    & negative ($\negg$) of width $k$             \\
$\IS{}{00}$           & cf. next column                          & $(\neg x_1), (x_1 \rightarrow x_2), (x_1 \lor \dots \lor x_n), n \geq 0$                               & IHS-B+                                      \\
$\IS{k}{00}$          & cf. next column                          & $(\neg x_1), (x_1 \rightarrow x_2), (x_1 \lor \dots \lor x_n), k \geq n \geq 0$                        & IHS-B+ of width $k$                         \\
$\IS{}{02}$           & cf. next column                          & $(\neg x_1), (x_1 \lor \dots \lor x_n), n \geq 0, (x_1 = x_2)$                                         & essentially positive ($\epos$)              \\
$\IS{k}{02}$          & cf. next column                          & $(\neg x_1), (x_1 \lor \dots \lor x_n), k \geq n \geq 0, (x_1 = x_2)$                                  & essentially positive ($\epos$) of width $k$ \\
$\IS{}{01}$           & cf. next column                          & $(x_1 \rightarrow x_2), (x_1 \lor \dots \lor x_n), n \geq 0$                                           & -                                           \\
$\IS{k}{01}$          & cf. next column                          & $(x_1 \rightarrow x_2), (x_1 \lor \dots \lor x_n), k \geq n \geq 0$                                    & -                                           \\
$\IS{}{0}$            & cf. next column                          & $(x_1 \lor \dots \lor x_n), n \geq 0, (x_1 = x_2)$                                                     & positive ($\pos$)                           \\
$\IS{k}{0}$           & cf. next column                          & $(x_1 \lor \dots \lor x_n), k \geq n \geq 0, (x_1 = x_2)$                                              & positive ($\pos$) of width $k$              \\
$\IR_2$               & $x_1,  \neg x_2$                         & $(x_1), (\neg x_1), (x_1 = x_2)$                                                                       & -                                           \\
$\IR_1$               & $x_1$                                    & $(x_1), (x_1 = x_2)$                                                                                   & -                                           \\
$\IR_0$               & $\neg x_1$                               & $(\neg x_1), (x_1 = x_2)$                                                                              & -                                           \\
$\IR$ ($\IBF$)        & $\emptyset$                              & $(x_1 = x_2)$                                                                                          & -                                           \\\bottomrule
\end{tabular}}	
\caption{Overview of bases \cite{BohlerRSV05} and clause descriptions \cite{NordhZ08} for co-clones, where \textrm{EVEN}$^4$ = $x_1 \oplus x_2 \oplus x_3 \oplus x_4 \oplus 1$.}%
\label{tab:bases}
\end{table*}

\section{Omitted Examples}

We provide a larger, self-contained example of a constraint language, the relationship to clauses, pp-definitions, and co-clones. 

\begin{example}
We begin by defining the following three Boolean relations.

\begin{enumerate}
    \item 
    $R = \{0,1\}^3 \setminus \{(1,1,0)\}$,
    \item 
    $S = \{(0)\}$,
    \item 
    $T = \{(1)\}$.
\end{enumerate}

Represented as clauses these three relations correspond to $(\neg x \lor \neg y \lor z)$, $(\neg x)$, and $(x)$. Hence, we take us the liberty to view $\Gamma = \{(\neg x \lor \neg y \lor z), (\neg x), (x)\}$ as the of relations $\{R, S, T\}$. Now, let us consider $\closneq{\Gamma}$, i.e., the set of relations definable from $\Gamma$ by primitive positive definitions without equality. We first remark that $(x \rightarrow y) \equiv R(x,x,y)$, or, written in clausal form, $(x \rightarrow y) \equiv (\neg x \lor \neg x \lor y)$. Via this auxilliary relation we can easily define equality as $(x = y) \equiv (x \rightarrow y) \land (y \rightarrow x)$. This implies that pp-definitions and equality-free pp-definitions over $\Gamma$ coincide and that $\closneq{\Gamma} = \clos{\Gamma}$. Furthermore, it is easy to see that we can define the 4-ary Horn-clause $(\neg x \lor \neg y \lor \neg z \lor v)$ as \[(\neg x \lor \neg y \lor \neg z \lor v) \equiv \exists w \,.\, (\neg x \lor \neg y \lor w) \lor (\neg w \lor \neg z \lor v).\] This definition can be repeated for higher arities and we conclude that $\cclone{\Gamma}$ is precisely the set of all relations definable as conjunctions of Horn-clauses. In co-clone terminology this set of relations is denoted by $\IE_2$ where the naming convention carries the following meaning: 
\begin{enumerate}
    \item 
    the subscript $2$ means that $\IE_2$ contains both $\{(0)\}$ and $\{(1)\}$,
    \item 
    $E_2$ is the corresponding algebraic object (a {\em clone}) generated by the Boolean function $\land$, and $\IE_2$ is then supposed can then be interpreted as the set of all Boolean relations {\em invariant} under $\land$. 
\end{enumerate}

One natural interpretation of the latter condition is that we for any propositional, conjunctive Horn-formula $\psi$, are guaranteed that the (component-wise) application $f \land g$ is a model of $\psi$ whenever $f$ and $g$ are both models of $\psi$. Due to Post's lattice (see Figure~\ref{fig:post-latticeAppendix}) similar descriptions are known for {\em all} possible sets of Boolean relations.
\end{example}

\section{Omitted Proofs}
\unitclauses*
\begin{proof}
Let $(\KB, H, M, x)$ be an instance of $\isfacet(\text{dualHorn})$.
We exhaustively apply unit propagation to $\KB$ and obtain a representation as a unit clause-free part $\KB'$ and a set of unit clauses $U$ such that $\KB \equiv \KB' \land U$ and $\var(\KB') \cap \var(U) = \emptyset$.
We output an arbitrary no-instance if one of the following cases applies.

\begin{enumerate}
    \item $\KB$ is unsatisfiable (no explanations can exist).
    \item An $m \in M$ occurs as $\neg m$ in $U$ (no explanations can exist).
    \item $x$ or $\neg x$ occurs in $U$ (then $x$ can not be a facet).
\end{enumerate}
 Otherwise, we map to the instance $(\KB', H\setminus U, M \setminus U, x)$ of $\isfacet(\text{dualHorn}^{\text{-}})$. It is easy to verify that $x$ is a facet in $(\KB', H\setminus U, M \setminus U)$ if and only if $x$ is a facet in $(\KB, H, M)$, provided none of the three cases above applies.
\end{proof}

\ubdualHorn*
\begin{proof}
We first observe that in any $\isfacet(\text{dualHorn}^{\text{-}})$ instance $(\KB, H, M, x)$ , we deal with clauses of type (1) $(x \rightarrow y)$, (2) $(x_1 \lor \dots \lor x_k)$ for $k \geq 2$, and (3) $(\neg x_0 \lor x_1 \lor \dots \lor x_k)$ for $k \geq 2$.
Note that $\KB$ is 1-valid. Consequently, there is an explanation if and only if $H$ is an explanation.
We observe further that to explain a single $m\in M$, a single $h\in H$ is always sufficient. Therefore, the 
algorithm in  Lemma~\ref{lem:facetImpP} is applicable, and the running time is still polynomial, since $h(m)$ and $M_x$ can still be computed in polynomial time.
\end{proof}

\NPverify*
\begin{proof}
Let ${\cal A}$ be a polynomial time algorithm for $\SAT(\Gamma^+)$. Let $(\KB, H, M)$ be an instance of $\PMABD(\Gamma)$ and $E \subseteq H$ a given explanation candidate. The following algorithm determines in polynomial time whether $E$  is an explanation or not.
\begin{algorithmic}[1]
    \STATE {\# \emph{determine whether $\KB \land E$ is consistent:}}
    \IF {not ${\cal A}(\KB \land E$)}
    \RETURN False
    \ENDIF
    
    \STATE {\# \emph{determine whether $\KB \land E \models M$}}
    \FOR {all $m \in M$}
      \IF {${\cal A}(\KB \land E \land \neg m$)}
      \RETURN False
      \ENDIF
    \ENDFOR
    \RETURN {True}
\end{algorithmic}

 Note that both $\KB \land E$ and $\KB \land E \land \neg m$ can be expressed as $\Gamma^+$-formulas. We conclude by observing that the algorithm calls ${\cal A}$ a linear number of times and thus runs in polynomial time.
\end{proof}

\InNp* 
\begin{proof}
Denote by ${\cal B}$ the algorithm from Lemma~\ref{lem:NPverify} to verify an explanation candidate in polynomial time. Let $(\KB, H, M, x)$ be an instance of $\isfacet(\Gamma)$. The algorithm for $\NP$-membership is as follows.

\begin{enumerate}
\item guess $E_1 \subseteq H$ such that $x \in E_1$,
\item guess $E_2 \subseteq H$ such that $x \notin E_2$,
\item use ${\cal B}$ to verify in polynomial time that $E_1$ is an explanation,
\item use ${\cal B}$ to verify in polynomial time that $E_1\setminus \{x\}$ is \emph{not} an explanation,
\item use ${\cal B}$ to verify in polynomial time that $E_2$ is an explanation.
\end{enumerate}

\end{proof}

\lemFacetT*
\begin{proof}
Let $(\KB, H, M, x)$ be an instance of $\isfacet(\Gamma \cup \{(x)\})$. If $x$ occurs as a unit clause $(x) \in \KB$, then it can not be a facet, %
as it is forced to be true in any model of $\KB$ and thus can always be removed from an explanation. Hence, the output is %
a fixed negative instance, e.g., $(y, x, y, x)$.
If $x$ does not occur as a unit clause, 
we can simulate unit clauses as follows. %
Identify all variables occurring in positive unit clauses with a single fresh variable $t$, then delete all positive unit clauses, and add $t$ to $H$ and to $M$. 
\end{proof}

\FacetFext*
\begin{proof}
By Lemma~\ref{lem:FacetF} we can reduce $\PMABD(\Gamma \cup \{(\neg x)\})$ to $\isfacet(\Gamma \cup \{x \rightarrow y\})$.
Note that $(x\rightarrow y)$ is 1-valid and Horn, and thus $(x\rightarrow y) \in \IE_1 \subseteq \II_1 \subseteq \clos{\Gamma}$. Since by Lemma~\ref{lem:express_equality}
$\clos{\Gamma} = \closneq{\Gamma}$, we obtain that 
$(x\rightarrow y) \in \closneq{\Gamma}$. Thus, $\Gamma \cup \{x \rightarrow y\} \subseteq \closneq{\Gamma}$, and we can apply Lemma~\ref{lem:baseIndneq} to obtain the desired reduction to $\isfacet(\Gamma)$.    
\end{proof}

\FacetFextB*
\begin{proof}
    We first consider the case where $\II_1 = \clos{\Gamma}$. Since $(\neg x) \notin \II_1$, we obtain that $\clos{\Gamma \cup \{(\neg x)\}} = \II_2 = \BR$. Thus, $\PMABD(\Gamma\cup \{(\neg x)\})$ is $\SigmaP$-complete \cite{NordhZ08}. 
    By Lemma~\ref{lem:FacetFext} we can reduce this problem to $\isfacet(\Gamma)$, providing the $\SigmaP$-hardness.

    Now consider the case $\IN \subseteq \clos{\Gamma}$. Since $(x) \notin \IN$, $(x) \notin \IN_2$, $(x)\notin \II_0$, $(x) \notin \II$, $(x) \in \II_1$, we obtain that $\clos{\Gamma \cup \{(x)\}} = \II_1$. Thus,  by the first case $\isfacet(\Gamma\cup \{(x)\})$ is $\SigmaP$-hard.
    By Lemma~\ref{lem:FacetT} we can reduce this problem to $\isfacet(\Gamma)$, providing the $\SigmaP$-hardness.
\end{proof}

\iehard*
\begin{proof}
    We first consider the case where $\IE_1 = \clos{\Gamma}$. Since $(\neg x) \notin \IE_1$, we obtain that $\clos{\Gamma \cup \{(\neg x)\}} = \IE_2$. Thus, $\PMABD(\Gamma\cup \{(\neg x)\})$ is $\NP$-complete \cite{NordhZ08}. 
    By Lemma~\ref{lem:FacetFext} we can reduce this problem to $\isfacet(\Gamma)$, providing the $\NP$-hardness.

    Now consider the case $\IE \subseteq \clos{\Gamma}$.
    Since $(x) \notin \IE$, $(x) \notin \IE_0$, $(x) \in \IE_1$, we obtain that $\clos{\Gamma \cup \{(x)\}} = \IE_1$. Thus, by the first case $\isfacet(\Gamma\cup \{(x)\})$ is $\NP$-hard.
    By Lemma~\ref{lem:FacetT} we can reduce this problem to $\isfacet(\Gamma)$, providing the $\NP$-hardness.
\end{proof}

\LemDivEq*
\begin{proof}
  The same reduction as in Lemma~\ref{lem:abdIsfacet} works, replacing $x$ by $k = 2$ in the target instance. That is, $(\KB, H, M) \mapsto (\KB', H', M', 2)$.
\end{proof}

\LemDivIN*
\begin{proof}
  The reduction of Lemma~\ref{lem:FacetF} works %
  for $\divABD$ by replacing $x$ with $k = 2$ in the target instance. Then, we proceed analogously to Lemma~\ref{lem:FacetFext} and Lemma~\ref{lem:FacetFextB}. %
\end{proof}

\LemDivImp*

\begin{proof}
Let $(\varphi, k)$ be an instance of $\divPosTwoSAT$, where $\varphi = \bigwedge_{i=1}^m C_i$ and each clause $C_i$ is positive and of size 2. Denote $\varphi$'s variables by $x_1, \dots, x_r$. Introduce fresh variables $c_1, \dots, c_m$ to represent the clauses. We construct an abduction instance $(\KB, H, M)$ as follows.
\begin{align*}
    \KB & = \bigwedge_{i=1}^m \bigwedge_{x \in C_i} (x \rightarrow c_i) \\
    H  &= \{x_1, \dots, x_r\} \\
    M  &= \{c_1, \dots, c_m\} \\
\end{align*}
It is easily seen that there is a one-to-one correspondence between models of $\varphi$ and explanations for $M$: if $\sigma$ is a model, then the corresponding explanation is given by $E_\sigma = \{x \in H \mid \sigma(x) = 1\}$. One further observes that the Hamming distance between two models $\sigma, \theta$ is exactly $|E_\sigma \Delta E_\theta|$.
\end{proof}

\begin{figure*}[htp]
    \centering
       \includegraphics[width=\textwidth]{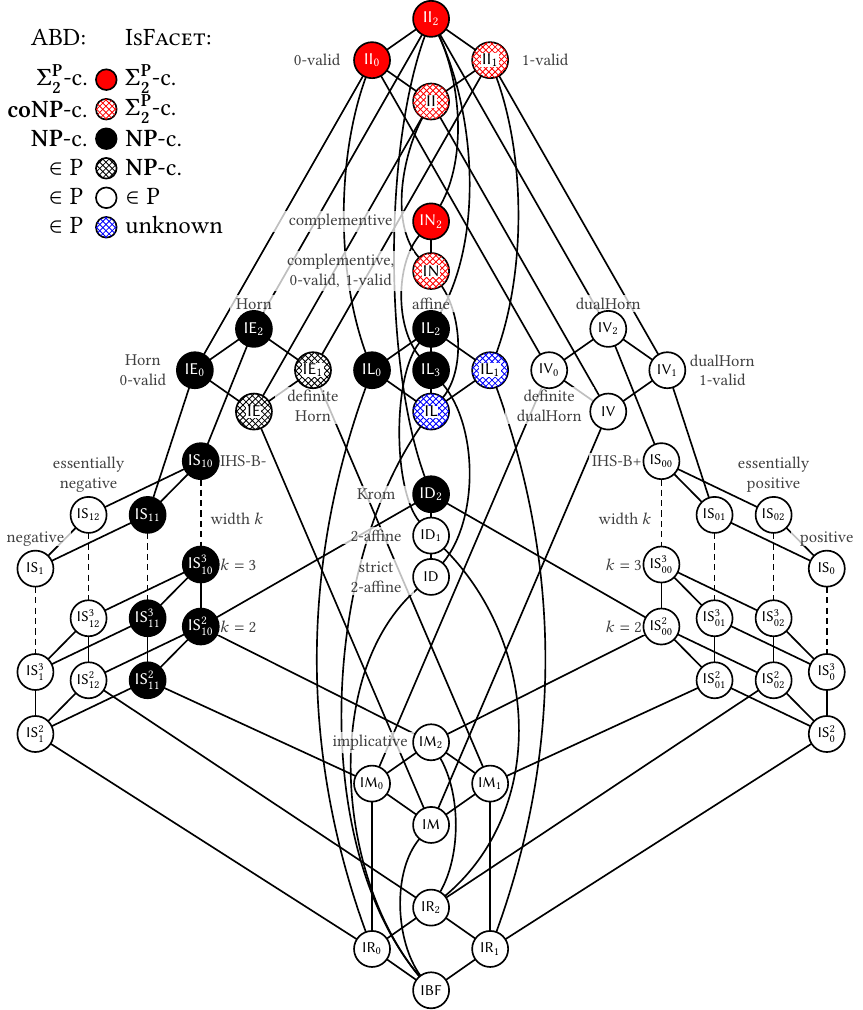}
    \caption{Illustration of complexity via Post's lattice.}
    \label{fig:post-latticeAppendix}
\end{figure*}

\lbAffine*
\begin{proof}
For the first statement, let $(\KB,H,M,k)$ be an instance of $\divABD(\text{2-affine})$. We reuse the cluster representation from Lemma~\ref{lem:width2affineP} and again
propagate all unit clauses as in Lemma~\ref{lem:unitclauses}.
Provided explanations exist, we can now construct two explanations of maximal distance from ``local'' to ``global''. For each $m_i \in M$ be $X_i$ the corresponding equivalence class, that is, $m_i \in X_i$. We 
construct two ``local'' explanations for an $m_i$ of maximal distance by picking for $E^{m_i}_1$ one representative from $X_i \cap H$, and for 
$E^{m_i}_2$ the remaining representatives (if any, otherwise $E^{m_i}_2 = E^{m_i}_1$). This maximizes $|E^{m_i}_1 \Delta E^{m_i}_2$| (within 
$X_i \cap H$). Two ``global'' explanations (of maximal distance within $\bigcup_{i = 1}^p X_i \cap H$) are obtained by the corresponding 
unions, that is, $E_1 = \bigcup_{i = 1}^p E^{m_i}_1$ and $E_2 = \bigcup_{i = 1}^p E^{m_i}_2$. Additional variables in $ H' := H 
\setminus \bigcup_{i = 1}^p X_i \cup Y_i$ may exist. Those are used to 
maximize the distance further. For every cluster $(X, Y)$ containing variables from $H'$: add to $E_1$ either all variables in $X\cap H$ or 
all in $Y\cap H$, depending on which set is larger.

For the second statement, 
    let $(\KB,H,M,k)$ be an instance of $\divABD(\text{EP})$. We first perform unit propagation exactly as in Lemma~\ref{lem:essentially_negative}. Now we have an instance that contains positive clauses of $size \geq 2$ and equality clauses.
    Again, exactly as in Lemma~\ref{lem:essentially_negative} we organize the variables into equivalence classes.
    Now we can construct two explanations of maximal distance (if explanations exist) by the same procedure as above. It is only simpler since all equivalence classes are independent this time. In other words, each cluster is of the form $(X, \emptyset)$.
\end{proof}

\end{document}